\documentclass[12pt,a4paper]{article}

\usepackage[top=2.5cm,bottom=2.5cm,left=3cm,right=3cm]{geometry}
\usepackage{amsmath,amssymb,amsthm,mathtools}
\usepackage{bm}
\usepackage{booktabs}
\usepackage{hyperref}
\usepackage{xcolor}
\usepackage{microtype}
\usepackage{natbib}
\usepackage{enumitem}
\usepackage{array}
\usepackage{longtable}
\usepackage{setspace}
\usepackage{fancyhdr}
\usepackage{tikz}
\usepackage{pgfplots}
\pgfplotsset{compat=1.18}
\usetikzlibrary{arrows.meta,positioning,shapes.geometric,calc}

\hypersetup{
  colorlinks = true,
  linkcolor = black,
  citecolor = black,
  urlcolor = black,
  pdftitle = {The Mathematics of AI Winters},
  pdfauthor = {Miquel Noguer i Alonso and David Pacheco Aznar}
}

\theoremstyle{plain}
\newtheorem{theorem}{Theorem}[section]
\newtheorem{proposition}[theorem]{Proposition}

\newtheorem{lemma}[theorem]{Lemma}
\theoremstyle{definition}
\newtheorem{definition}[theorem]{Definition}
\newtheorem{remark}[theorem]{Remark}

\DeclareMathOperator*{\argmin}{arg\,min}

\DeclareMathOperator{\E}{\mathbb{E}}
\DeclareMathOperator{\Var}{Var}

\DeclareMathOperator{\sign}{sign}
\DeclareMathOperator{\VCdim}{VCdim}
\DeclareMathOperator{\tr}{tr}

\newcommand{\R}{\mathbb{R}}
\newcommand{\N}{\mathbb{N}}
\newcommand{\Prob}{\mathbb{P}}
\newcommand{\norm}[1]{\left\lVert #1 \right\rVert}
\newcommand{\abs}[1]{\left\lvert #1 \right\rvert}
\newcommand{\inner}[2]{\langle #1,\,#2 \rangle}

\newcommand{\calH}{\mathcal{H}}
\newcommand{\calX}{\mathcal{X}}
\newcommand{\calY}{\mathcal{Y}}
\newcommand{\calD}{\mathcal{D}}
\newcommand{\calL}{\mathcal{L}}

\newcommand{\calA}{\mathcal{A}}

\begin{document}

\begin{center}
{\LARGE \bfseries The Mathematics of AI Winters\par}
\vspace{0.5em}
{\large A Mathematical Taxonomy of Paradigm Fragility in AI Winter\par}
\vspace{1em}
{\normalsize Miquel Noguer i Alonso (AIFI), David Pacheco Aznar (Staq.io)\par}
\vspace{1em}
{\normalsize June 2026\par}
\end{center}

\begin{abstract}
The two major periods of reduced funding and confidence in artificial intelligence research, usually described as the first and second AI winters, are often explained in terms of engineering failure, commercial disappointment, and inflated expectations. This article develops a complementary thesis. It argues that the dominant paradigms of the relevant periods also encountered genuine formal barriers: limitations of representation, optimisation, computational complexity, statistical learnability, and high-dimensional approximation. The contribution is synthetic rather than archival. We do not claim that particular theorems mechanically caused the winters. Rather, we show that several central disappointments of early AI were aligned with mathematically precise bottlenecks. We analyse those bottlenecks through the perceptron impossibility results associated with \citet{minsky1969}, the complexity-theoretic hardness of exact neural-network training \citep{blum1992}, minimax rates for nonparametric estimation in high dimension \citep{stone1982}, vanishing-gradient analyses \citep{hochreiter1991,bengio1994}, and classical statistical learning theory \citep{vapnik1971,valiant1984,blumer1989}. We then relate these barriers to the later breakthroughs that mitigated, rather than eliminated, them.

\vspace{0.5em}
\end{abstract}
\newpage
\tableofcontents
\newpage

\section{Introduction}

The history of artificial intelligence is marked by repeated cycles of expansion and retrenchment. Two episodes are conventionally identified as AI winters: the first, associated with the mid-1970s, followed the critical reception of ambitious symbolic and perceptron-based programmes, most visibly after the Lighthill Report \citep{lighthill1973}, the philosophical critique of \citet{dreyfus1972}, and the reception of \emph{Perceptrons} \citep{minsky1969}; the second, associated with the late 1980s and early 1990s, followed the collapse of the expert-systems market --- whose intellectual programme is best represented by \citet{feigenbaum1977} and \citet{buchanan1984} --- and the limited practical performance of shallow connectionist systems. The standard explanations emphasise commercial overpromising, insufficient hardware, brittle engineering, and institutional disappointment. Those explanations are important, but they are not the whole story.

This article studies a narrower claim. It argues that the dominant paradigms of those periods also faced mathematically identifiable bottlenecks. Some were representational, as in the inability of single-layer threshold models to capture non-linearly separable relations. Some were computational, as in the combinatorial explosion of symbolic search and the worst-case hardness of exact training. Some were statistical, as in the weakness of then-available generalisation guarantees for flexible models. Others were optimisation-theoretic, as in the instability of deep gradient propagation under saturating nonlinearities. The aim is therefore not to reduce intellectual history to theorem-proving. It is to show that the disappointments of those eras were, in part, consistent with formal limits rather than merely with poor execution.

The contribution is mainly synthetic. None of the underlying theorems is new. What is new is their integration into a unified mathematical interpretation of AI winters.

\paragraph{Methodological clarification.} At the level of history, our claim is one of compatibility and necessary vulnerability, not monocausal explanation. At the level of mathematics, the novelty of the paper lies not in new lower bounds but in a budgeted synthesis: representational limits, optimisation hardness, and finite-sample fragility are analysed as jointly binding constraints under shared data and compute budgets. This narrower claim is also the stronger one, because it is precisely the claim for which the cited theory can supply genuine support.

Throughout, we distinguish three notions that are often conflated in informal discussion: representational capacity, learnability from finite samples, and tractable trainability. A model class may be expressive without being efficiently trainable; it may be trainable without carrying non-vacuous finite-sample guarantees; and it may possess asymptotic guarantees while remaining practically unusable at realistic scales. This separation is crucial for interpreting both the failures of early AI and the later successes of deep learning.

A further qualification is necessary. The theorems discussed below did not by themselves cause the winters. The historical episodes were also shaped by economics, hardware, programming practice, benchmark choice, and the sociology of research communities. Our claim is therefore deliberately limited: formal barriers materially contributed to the fragility of the leading paradigms, and later progress became durable only when those barriers were mitigated by new mathematics, algorithms, and architectures.

The article proceeds as follows. Section~\ref{sec:pac} introduces the formal learning framework. Section~\ref{sec:perceptron} studies representational barriers and the classical perceptron results. Section~\ref{sec:complexity} examines computational hardness, combinatorial explosion, and high-dimensional estimation. Section~\ref{sec:gradients} derives the vanishing-gradient phenomenon. Section~\ref{sec:generalization} reviews the statistical barriers associated with generalisation. Section~\ref{sec:resolution} considers the mathematical developments that mitigated these difficulties. Section~\ref{sec:taxonomy} summarises the argument in comparative form, and Section~\ref{sec:conclusion} concludes.

\paragraph{Status of results.} To avoid overstating what is proved here, we classify the
mathematical content of the article into four categories.
\emph{(i) Classical theorems cited from the literature}: the perceptron impossibility
results, the Blum--Rivest NP-hardness theorem, the Stone minimax rate, the
vanishing-gradient analyses, and the foundational PAC/VC bounds. These are used as
given and credited to their original sources.
\emph{(ii) Framework synthesis}: Definition~\ref{def:barrier} (the barrier triple)
and Proposition~\ref{thm:fragility} (the Joint-Binding condition) organise the
classical results into a single joint-binding condition parameterised by a data and
compute budget. The proposition is a structural reorganisation of the preceding
sections, not a new lower bound; its proof is labelled a \emph{synthesis} for this
reason.
\emph{(iii) Interpretive remarks}: the connections drawn between the barrier triple
and modern developments (scaling laws, neural collapse, lottery tickets, implicit
bias) are post-hoc lenses, not historical-causal claims.
\emph{(iv) Historical context}: all attributions of winter episodes to formal
barriers are stated as necessary-condition arguments, never as sufficient ones; the
actual historical interventions (multilayer architectures, non-saturating
activations, improved initialisation, large corpora, hardware scaling) are named
explicitly where relevant.

\subsection{A Formal Notion of an AI Winter Barrier}

The central thesis of this article can be stated more formally by identifying the
mathematical structure underlying paradigm fragility.

\begin{definition}[AI Winter Barrier]\label{def:barrier}
Let $\mathcal{A}$ be a class of AI architectures and $\mathcal{T}$ a class of
learning or reasoning tasks. A \emph{formal AI winter barrier} is a triple
\[
\mathcal{B}(\mathcal{A},\mathcal{T})
= \bigl(R(\mathcal{A},\mathcal{T}),\;
         C(\mathcal{A},\mathcal{T}),\;
         S(\mathcal{A},\mathcal{T})\bigr),
\]
where:
\begin{enumerate}[label=(\roman*)]
  \item $R$ is a \emph{representational constraint} limiting the set of functions
        or relations realisable by $\mathcal{A}$;
  \item $C$ is a \emph{computational or optimisation constraint} implying that
        training, inference, or search over $\mathcal{A}$ is intractable in the
        worst case or scales prohibitively in practice;
  \item $S$ is a \emph{statistical constraint} under which finite-sample guarantees
        become vacuous or sample requirements become unrealistic at relevant scales.
\end{enumerate}
We say that a paradigm operates in a \emph{winter regime} if at least two of
these constraints bind simultaneously under realistic data and compute budgets.
\end{definition}

\begin{remark}
We will make the joint-binding condition precise in
Proposition~\ref{thm:fragility}, where the data budget $n$ and compute budget
$\kappa$ appear as explicit parameters and each constraint is tied to a
quantitative obstruction. Definition~\ref{def:barrier} itself remains at a
structural level: it fixes vocabulary, not thresholds.
\end{remark}

\begin{remark}
This definition is not intended as a historical law. It is a structural device that
organises the various theorems discussed below into a common analytical framework.
Its purpose is to make precise what it means for an AI paradigm to be mathematically
fragile.
\end{remark}

\paragraph{Central insight.} AI winters can be interpreted as periods in which the dominant paradigms faced a joint failure regime: representational restrictions, computational intractability, and weak statistical control were not isolated problems but interacting constraints. The formal synthesis is given in Proposition~\ref{thm:fragility} (Section~\ref{sec:taxonomy}).

\paragraph{Notation.} We write $\R$ for the reals, $\N$ for the natural numbers,
$[n] = \{1,\ldots,n\}$. Vectors are bold lowercase ($\mathbf{x}$), matrices bold
uppercase ($\mathbf{W}$). $\norm{\cdot}$ denotes the Euclidean norm unless otherwise
specified. $\Prob[\cdot]$ and $\E[\cdot]$ are probability and expectation with
respect to a distribution clear from context.

\section{The Formal Learning Framework}\label{sec:pac}

\subsection{The PAC Model}

We work throughout in the Probably Approximately Correct (PAC) learning
framework of \citet{valiant1984}.

\begin{definition}[Learning Problem]
  Let $\calX$ be an instance space, $\calY$ an output space, and $\calD$ an
  unknown joint distribution over $\calX \times \calY$. Given a hypothesis
  class $\calH \subseteq \calY^{\calX}$ and a loss function
  $\ell: \calY \times \calY \to \R_{\geq 0}$, the \emph{true risk} of
  $h \in \calH$ is
  \[
    L_{\calD}(h) \;=\; \E_{(x,y) \sim \calD}[\ell(h(x), y)],
  \]
  and the \emph{empirical risk} on a sample $S = \{(x_i, y_i)\}_{i=1}^m$ is
  \[
    L_S(h) \;=\; \frac{1}{m}\sum_{i=1}^m \ell(h(x_i), y_i).
  \]
\end{definition}

\begin{definition}[Realizable PAC Learnability]
  A binary hypothesis class $\calH$ is \emph{PAC learnable in the realizable
  setting} if there exists an algorithm $\calA$ and a function
  $m^{\mathrm{real}}_{\calH}:(0,1)^2 \to \N$ such that whenever the data-generating
  distribution $\calD$ satisfies $\inf_{h \in \calH} L_{\calD}(h) = 0$, then for
  every $\varepsilon,\delta \in (0,1)$ and every $m \geq
  m^{\mathrm{real}}_{\calH}(\varepsilon,\delta)$,
  \[
    \Prob_{S \sim \calD^m}\!\left[L_{\calD}(\calA(S)) \le \varepsilon\right]
    \ge 1-\delta.
  \]
\end{definition}

\begin{definition}[Agnostic PAC Learnability]
  A binary hypothesis class $\calH$ is \emph{agnostically PAC learnable} if
  there exists an algorithm $\calA$ and a function
  $m^{\mathrm{agn}}_{\calH}:(0,1)^2 \to \N$ such that for every distribution
  $\calD$ over $\calX \times \{0,1\}$, every $\varepsilon,\delta \in (0,1)$,
  and every $m \ge m^{\mathrm{agn}}_{\calH}(\varepsilon,\delta)$,
  \[
    \Prob_{S \sim \calD^m}\!\left[L_{\calD}(\calA(S)) \le
    \inf_{h \in \calH} L_{\calD}(h) + \varepsilon\right] \ge 1-\delta.
  \]
\end{definition}

The distinction matters. Realizable PAC learning assumes that the target relation
is representable inside $\calH$, whereas agnostic PAC learning competes only with
best-in-class risk. In modern learning theory the agnostic setting is the more
robust formulation, but the cleanest classical VC equivalence is usually first
stated in the realizable case; the agnostic characterization proceeds through
uniform convergence and empirical risk minimisation \citep{blumer1989,ehrenfeucht1989,shalev2014}.

\subsection{VC Dimension and the Fundamental Theorem}

\begin{definition}[Shattering and VC Dimension]
  A set $\mathcal{C} = \{x_1,\ldots,x_m\} \subset \calX$ is \emph{shattered} by
  $\calH$ if for every labelling $\mathbf{y} \in \{0,1\}^m$ there exists
  $h \in \calH$ with $h(x_i) = y_i$ for all $i$. The
  \emph{Vapnik--Chervonenkis dimension} is
  \[
    \VCdim(\calH) \;=\; \sup\bigl\{m \in \N : \exists\, \mathcal{C} \subset \calX,
    \; \abs{\mathcal{C}} = m,\; \calH \text{ shatters } \mathcal{C}\bigr\}.
  \]
\end{definition}

\begin{theorem}[Fundamental Theorem in the Realizable Case]\label{thm:pac}
\citep{vapnik1971,blumer1989,ehrenfeucht1989}
  For binary classification in the realizable setting, a hypothesis class
  $\calH$ is PAC learnable if and only if $d = \VCdim(\calH) < \infty$.
  Moreover, the realizable sample complexity satisfies\footnote{The $\log(1/\varepsilon)$
  gap between the classical upper and lower bounds was closed by \citet{hanneke2016},
  who established the tight rate $m^{\mathrm{real}}_{\calH}(\varepsilon,\delta) =
  \Theta((d + \log(1/\delta))/\varepsilon)$ via a one-inclusion-graph majority-vote
  algorithm. The classical bounds above remain the historically relevant statement
  for the period covered in this article.}
  \[
    C_1\,\frac{d + \log(1/\delta)}{\varepsilon}
    \;\leq\; m^{\mathrm{real}}_{\calH}(\varepsilon,\delta)
    \;\leq\; C_2\,\frac{d\log(1/\varepsilon) + \log(1/\delta)}{\varepsilon}
  \]
  for universal constants $C_1,C_2>0$.
\end{theorem}

In the agnostic setting, finite VC dimension still characterises learnability,
but the standard sample-complexity scale becomes
\[
  m^{\mathrm{agn}}_{\calH}(\varepsilon,\delta)
  = \Theta\!\left(\frac{d + \log(1/\delta)}{\varepsilon^2}\right),
\]
with the proof proceeding through uniform convergence and empirical risk
minimisation rather than through the realizable reduction alone
\citep{blumer1989,ehrenfeucht1989,shalev2014}. The key message is unchanged:
classes with $\VCdim(\calH) \gg m$ do not admit non-vacuous finite-sample
control.

\begin{lemma}[Sauer--Shelah]
  Let $d = \VCdim(\calH) < \infty$. The growth function
  \[
    \Pi_{\calH}(m) \;=\; \max_{\mathcal{C} \subset \calX,\,\abs{\mathcal{C}}=m}
    \abs{\calH_{\mathcal{C}}}
  \]
  satisfies
  \[
    \Pi_{\calH}(m) \;\leq\; \sum_{i=0}^{d} \binom{m}{i}
    \;\leq\; \left(\frac{em}{d}\right)^d.
  \]
\end{lemma}

\begin{proof}
  By induction on $m + d$. For $m = 0$ or $d = 0$ the bound is trivial.
  Partition any $m$-element set by taking $x_m$ and noting that restrictions
  of $\calH$ to the remaining $m-1$ elements satisfy both $\Pi_{\calH}(m-1) \leq
  \sum_{i=0}^{d}\binom{m-1}{i}$ and that pairs forced by $x_m$ satisfy
  $\Pi_{\calH'}(m-1) \leq \sum_{i=0}^{d-1}\binom{m-1}{i}$. Adding via
  Pascal's identity gives the result.
\end{proof}

\section{The Perceptron Crisis: Representational Barriers}\label{sec:perceptron}

\subsection{The Rosenblatt Perceptron}

The perceptron~\citep{rosenblatt1958} defines a linear classifier
\[
  f_{\mathbf{w},b}(\mathbf{x}) \;=\; \sign\!\bigl(\inner{\mathbf{w}}{\mathbf{x}} + b\bigr),
  \qquad \mathbf{w} \in \R^d,\; b \in \R.
\]
The perceptron algorithm, introduced by Rosenblatt, was later placed on firm mathematical ground by Novikoff's convergence theorem, which established that if the training data is linearly
separable, the perceptron learning rule converges in finite steps.

\begin{theorem}[Perceptron Convergence]\label{thm:perceptron}
\citep{novikoff1962}
  Let $\{(\mathbf{x}_i, y_i)\}_{i=1}^m \subset \R^d \times \{-1,+1\}$ be
  linearly separable with margin
  $\gamma = \min_i y_i \inner{\mathbf{w}^*}{\mathbf{x}_i}/\norm{\mathbf{w}^*} > 0$
  and $\norm{\mathbf{x}_i} \leq R$ for all $i$. Then the Perceptron algorithm
  makes at most
  \[
    T \;\leq\; \left(\frac{R}{\gamma}\right)^2
  \]
  mistakes before converging to a correct classifier.
\end{theorem}

\begin{proof}
  Let $\mathbf{w}_0 = \mathbf{0}$ and suppose mistake $t$ occurs at
  $(\mathbf{x}_{i_t}, y_{i_t})$, updating $\mathbf{w}_{t+1} = \mathbf{w}_t + y_{i_t}\mathbf{x}_{i_t}$.
  For the inner product with $\mathbf{w}^*$:
  \[
    \inner{\mathbf{w}_{t+1}}{\mathbf{w}^*} = \inner{\mathbf{w}_t}{\mathbf{w}^*}
    + y_{i_t}\inner{\mathbf{x}_{i_t}}{\mathbf{w}^*} \geq \inner{\mathbf{w}_t}{\mathbf{w}^*} + \gamma\norm{\mathbf{w}^*}.
  \]
  Thus $\inner{\mathbf{w}_T}{\mathbf{w}^*} \geq T\gamma\norm{\mathbf{w}^*}$.
  For the squared norm: since $y_{i_t}\inner{\mathbf{w}_t}{\mathbf{x}_{i_t}} \leq 0$
  on a mistake,
  \[
    \norm{\mathbf{w}_{t+1}}^2 = \norm{\mathbf{w}_t}^2 + 2y_{i_t}\inner{\mathbf{w}_t}{\mathbf{x}_{i_t}}
    + \norm{\mathbf{x}_{i_t}}^2 \leq \norm{\mathbf{w}_t}^2 + R^2,
  \]
  giving $\norm{\mathbf{w}_T}^2 \leq TR^2$. Combining the two bounds via
  Cauchy--Schwarz,
  \[
    T^2\gamma^2\norm{\mathbf{w}^*}^2
    \;\leq\; \inner{\mathbf{w}_T}{\mathbf{w}^*}^2
    \;\leq\; \norm{\mathbf{w}_T}^2\norm{\mathbf{w}^*}^2
    \;\leq\; TR^2\norm{\mathbf{w}^*}^2,
  \]
  and dividing through by $T\gamma^2\norm{\mathbf{w}^*}^2$ yields
  $T \leq (R/\gamma)^2$.
\end{proof}

\subsection{The Minsky--Papert Impossibility}

Despite this guarantee, the classical critique associated with \citet{minsky1969} made clear that single-layer perceptrons are restricted to linearly separable decision rules. The XOR example is the standard elementary illustration. It should not be read as proving that multilayer networks are impossible, but rather as identifying a genuine representational bottleneck for the architectural class then most prominently discussed.

\begin{theorem}[XOR Impossibility]\label{thm:xor}
\citep{minsky1969}
  The XOR function $f:\{0,1\}^2 \to \{-1,+1\}$ defined by
  $f(x_1,x_2) = 1 \iff x_1 \neq x_2$ is not linearly separable. That is,
  there exist no $w_1, w_2, b \in \R$ such that
  \[
    \sign(w_1 x_1 + w_2 x_2 + b) = f(x_1, x_2)
    \quad \forall\, (x_1,x_2) \in \{0,1\}^2.
  \]
\end{theorem}

\begin{proof}
  Suppose for contradiction such $w_1, w_2, b$ exist. The four conditions are:
  \begin{align}
    b &< 0 \label{eq:xor1} \\
    w_1 + b &> 0 \label{eq:xor2} \\
    w_2 + b &> 0 \label{eq:xor3} \\
    w_1 + w_2 + b &< 0. \label{eq:xor4}
  \end{align}
  Adding \eqref{eq:xor2} and \eqref{eq:xor3}: $w_1 + w_2 + 2b > 0$, so
  $w_1 + w_2 > -2b > 0$ (using \eqref{eq:xor1}). Therefore
  $w_1 + w_2 + b > -2b + b = -b > 0$, contradicting \eqref{eq:xor4}.
\end{proof}

The XOR example is the most elementary case of a broader phenomenon. Minsky and
Papert systematically characterised which Boolean functions are computable by a
single perceptron, showing the class to be severely limited. Crucially, they
demonstrated that adding more neurons in a single layer --- while possibly
increasing the number of computable functions --- could not overcome the
fundamental linear-separability constraint.

\subsection{VC Dimension of Linear Classifiers}

\begin{theorem}[VC Dimension of Halfspaces]
  For the class of halfspace classifiers
  $\calH = \{(\mathbf{x} \mapsto \sign(\inner{\mathbf{w}}{\mathbf{x}}+b))
  : \mathbf{w} \in \R^d, b \in \R\}$:
  \[
    \VCdim(\calH) = d + 1.
  \]
\end{theorem}

\begin{proof}[Proof sketch]
  \emph{Lower bound ($\VCdim \ge d+1$):} Consider the $d+1$ affinely independent
  points $\{\mathbf{0},\mathbf{e}_1,\ldots,\mathbf{e}_d\} \subset \R^d$. For any
  target labelling $(y_0,y_1,\ldots,y_d) \in \{-1,+1\}^{d+1}$, take
  $b = \tfrac{1}{2}y_0$ and, for $i = 1,\ldots,d$, set $w_i = y_i - b$. Then
  $\sign(b) = y_0$ and $\sign(w_i + b) = \sign(y_i) = y_i$ for each $i$, so every
  labelling is realised.

  \emph{Upper bound ($\VCdim \le d+1$):} By Radon's theorem, any set of $d+2$
  points in $\R^d$ admits a partition into two non-empty subsets $A,B$ whose
  convex hulls intersect. Label the points in $A$ by $+1$ and those in $B$ by $-1$.
  If a halfspace separated this labelling, convexity would force the intersection
  point of $\operatorname{conv}(A)$ and $\operatorname{conv}(B)$ to be assigned both
  signs, a contradiction. Hence no halfspace can shatter $d+2$ points.
\end{proof}

The consequence is fundamental: the realizable PAC sample complexity for halfspace
classifiers in $\R^d$ is linear in $d$, while the agnostic rate is of order
$d/\varepsilon^2$. For XOR in $\R^2$, under the uniform distribution on the four
points of $\{0,1\}^2$, the minimum classification error of any halfspace is $1/4$.
Thus the best attainable risk inside the class is strictly positive, a
class-specific irreducible error arising from insufficient representational power.

\section{Computational Complexity Barriers}\label{sec:complexity}

\subsection{NP-Hardness of Neural Network Training}

Even if an architecture has sufficient representational power, one must still
\emph{find} the weights. This problem is provably hard.

\begin{theorem}[Blum--Rivest Hardness]\label{thm:blumrivest}
\citep{blum1992}
  The following decision problem is NP-complete: given a set of labelled examples
  $\{(\mathbf{x}_i, y_i)\}_{i=1}^m \subset \{0,1\}^n \times \{0,1\}$, decide
  whether there exist weights $\mathbf{W}$ for a neural network with $2$ hidden
  units and one output unit (all with threshold activations) that correctly
  classifies all examples.
\end{theorem}

The proof proceeds by reduction from the \textsc{3-Dimensional Matching} problem,
which is NP-complete. The key insight is that two hidden threshold units can encode
two hyperplanes whose intersection region has sufficient combinatorial richness to
simulate Boolean constraints.

\begin{remark}[What the hardness result does and does not imply]
Theorem~\ref{thm:blumrivest} is a worst-case statement. It does not imply that every practically arising instance of network training is hard. Its relevance is more precise: representational sufficiency does not by itself entail efficient trainability. In particular, the existence of weights realising a target relation gives no general reason to expect that a reliable optimisation procedure will find them within realistic resource budgets.
\end{remark}

\subsection{The Lighthill Report and Combinatorial Explosion}

The Lighthill Report~\citep{lighthill1973} identified what it termed
``combinatorial explosion'' as the central failure mode of AI systems. In modern
mathematical language, this is the exponential blowup of the search space in tree-
and graph-based reasoning.

For a state-space search tree with branching factor $b$ and depth $d$, the number
of nodes at depth $d$ is $b^d$. The total search space is
\[
  \abs{\text{Search Tree}} = \sum_{k=0}^{d} b^k = \frac{b^{d+1}-1}{b-1} = \Theta(b^d).
\]
For the game of chess with $b \approx 35$ and $d \approx 80$, this is
$35^{80} \approx 10^{123}$ — far beyond any computational budget. Alpha-beta pruning
reduces this to $\Theta(b^{d/2})$ in the best case, still $10^{61}$ for chess.

Expert systems faced an analogous problem. In the worst case, matching rule antecedents against a working memory of $p$ atomic propositions involves a search over subsets of conditions whose size can grow as $2^p$; the precise behaviour depends on the match algorithm and on how many of the $n$ rules are simultaneously active, so the bound should be read as an illustrative worst case rather than a tight complexity figure.
The \emph{frame problem}~\citep{mccarthy1969} --- the difficulty of representing what
does \emph{not} change after an action --- is a further source of knowledge-base
bloat: one natural representation requires $\mathcal{O}(a \cdot f)$ persistence
axioms for $a$ actions and $f$ facts. As both grow, engineering and maintaining
such systems becomes increasingly costly, though the precise complexity depends on
the representation language and inference strategy; the $\mathcal{O}(a^2 f^2)$
figure sometimes cited informally reflects the cross-product of action--fact pairs
and is not a formal worst-case theorem.

\subsection{The Curse of Dimensionality}

\citet{bellman1961} first articulated the curse of dimensionality in the context of
dynamic programming. The phenomenon manifests across all areas of machine learning.

\begin{theorem}[Stone's Minimax Rate]\label{thm:stone}
\citep{stone1982}
  Let $\Sigma(s, L; [0,1]^d)$ denote the H\"{o}lder class of functions on
  $[0,1]^d$ whose $\lfloor s\rfloor$-th order partial derivatives are
  $(s - \lfloor s\rfloor)$-H\"{o}lder continuous with constant $L$. The
  minimax optimal estimation error over $m$ i.i.d.\ samples satisfies\footnote{The
  same exponent $-2s/(2s+d)$ extends to Sobolev classes
  $\mathcal{W}^{s,2}([0,1]^d)$ via wavelet-projection arguments
  \citep{donoho1998}, and more generally to Besov classes; the underlying
  obstruction --- the curse of dimensionality --- is identical across these scales.}
  \[
    \inf_{\hat{f}}\; \sup_{f \in \Sigma(s,L)} \E\!\left[\norm{\hat{f} - f}_{L^2}^2\right]
    \;\asymp\; m^{-2s/(2s+d)}.
  \]
\end{theorem}

This rate has a stark implication for sample requirements. Suppose we target
mean squared $L^2$ error at most $\delta$, i.e.\
$\E[\|\hat f - f\|_{L^2}^2] \leq \delta$. Setting
$m^{-2s/(2s+d)} \lesssim \delta$ and solving for $m$ gives
\[
  m \;=\; \Omega\!\left(\delta^{-(2s+d)/(2s)}\right).
\]
Parametrising by the root-MSE $\varepsilon = \sqrt{\delta}$, the exponent becomes
$(2s+d)/s$ in $\varepsilon$: $m = \Omega(\varepsilon^{-(2s+d)/s})$.
For $s = 2$ and $d = 10$ this exponent is $14/2 = 7$, giving
$m = \Omega(\varepsilon^{-7})$; for $d = 100$ it is $104/2 = 52$, giving
$m = \Omega(\varepsilon^{-52})$.

To see at what dimension the minimax sample requirement exceeds a budget of $m_0$
samples (for fixed $\varepsilon$ and $s$), invert $m_0 = \varepsilon^{-(2s+d)/s}$
to obtain
\[
  d^* \;=\; s\,\frac{\log m_0}{\log(1/\varepsilon)} - 2s.
\]
For $s = 2$ and $m_0 = 10^4$: with $\varepsilon = 0.1$ (ten-percent root-MSE) one gets
$d^* = 4$; with $\varepsilon = 0.01$ one gets $d^* = 0$, meaning the budget is
already insufficient for the one-dimensional problem at that accuracy. The precise
threshold depends sensitively on the target accuracy, but the qualitative conclusion
is unambiguous: even modest dimension exhausts any practical sample budget.

\begin{proposition}[Volume Concentration]
  In the unit ball $B_1^d \subset \R^d$, for any $\varepsilon \in (0,1)$:
  \[
    \frac{\mathrm{Vol}(B_1^d) - \mathrm{Vol}(B_{1-\varepsilon}^d)}{\mathrm{Vol}(B_1^d)}
    = 1 - (1-\varepsilon)^d \;\xrightarrow{d \to \infty}\; 1.
  \]
  That is, \emph{all volume concentrates near the surface} in high dimensions.
\end{proposition}

\begin{proof}
  The volume of a $d$-dimensional ball of radius $r$ is $V_d r^d$ where
  $V_d = \pi^{d/2}/\Gamma(d/2+1)$. Thus
  $\mathrm{Vol}(B_{1-\varepsilon}^d) = V_d(1-\varepsilon)^d$.
  Dividing: $1 - (1-\varepsilon)^d \to 1$ since $(1-\varepsilon)^d \to 0$
  exponentially for any fixed $\varepsilon > 0$.
\end{proof}

This means that in high dimensions, nearest-neighbour methods, kernel density
estimators, and all other methods that rely on local geometry effectively break
down. It is important to distinguish two complementary phenomena: the
\emph{volume concentration} established above, and the related but distinct
\emph{distance concentration}.

\begin{remark}[Distance Concentration]\label{rem:dist-conc}
  For i.i.d.\ $\mathbf{x}, \mathbf{y} \sim \mathcal{N}(\mathbf{0}, \mathbf{I}_d)$,
  the strong law of large numbers gives $\norm{\mathbf{x}}^2/d \to 1$ and
  $\norm{\mathbf{x}-\mathbf{y}}^2/(2d) \to 1$ almost surely as $d \to \infty$.
  More precisely, the Laurent--Massart concentration inequality for chi-squared
  random variables~\citep{laurent2000,vershynin2018} yields, for any
  $t \in (0,1)$,
  \[
    \Prob\!\bigl[\,|\norm{\mathbf{x}}^2/d - 1| > t\,\bigr]
    \;\leq\; 2\exp(-dt^2/8).
  \]
  A union bound over the $\binom{m}{2}$ pairs then yields the stronger statement that if $m=m(d)$ satisfies $\log m = o(d)$, the relative spread of all pairwise distances converges to zero in probability:
  \[
    \frac{\max_{i \neq j}\norm{\mathbf{x}_i - \mathbf{x}_j} - \min_{i \neq j}\norm{\mathbf{x}_i - \mathbf{x}_j}}
         {\min_{i \neq j}\norm{\mathbf{x}_i - \mathbf{x}_j}}
    \;\xrightarrow[d \to \infty]{\mathbb P}\; 0.
  \]
  Thus, whenever sample size grows subexponentially in the dimension, the geometry of pairwise distances becomes nearly degenerate, weakening the distinction between ``near'' and ``far'' for nearest-neighbour and kernel-based methods~\citep{beyer1999}. This is distinct from the volume-concentration proposition above, which concerns volume distribution within a ball; together, the two results give a complete picture of geometric degeneracy in high dimensions.
\end{remark}

\section{The Vanishing-Gradient Problem}\label{sec:gradients}

\subsection{Backpropagation and Gradient Flow}

Backpropagation~\citep{rumelhart1986} provided an efficient method to compute
gradients in multilayer networks. For a network with $L$ layers:
\begin{align}
  \mathbf{z}^{(\ell)} &= \mathbf{W}^{(\ell)}\mathbf{a}^{(\ell-1)} + \mathbf{b}^{(\ell)},
  \qquad \ell = 1, \ldots, L, \\
  \mathbf{a}^{(\ell)} &= \sigma(\mathbf{z}^{(\ell)}),
\end{align}
where $\sigma$ is applied componentwise. The per-layer Jacobian of activations is
\[
  \mathbf{J}^{(\ell)} \;:=\; \frac{\partial \mathbf{a}^{(\ell)}}{\partial \mathbf{a}^{(\ell-1)}}
  \;=\; \operatorname{diag}\!\bigl(\sigma'(\mathbf{z}^{(\ell)})\bigr)\,\mathbf{W}^{(\ell)},
  \qquad \ell = 2,\ldots,L,
\]
and the gradient of the loss with respect to the first-layer activations is obtained
by repeated application of the chain rule, composing Jacobians from output to input:
\begin{equation}\label{eq:chain}
  \frac{\partial \calL}{\partial \mathbf{a}^{(1)}}
  \;=\;
  \underbrace{\frac{\partial \calL}{\partial \mathbf{a}^{(L)}}}_{\boldsymbol{\delta}^{(L)\top}}
  \,\mathbf{J}^{(L)}\,\mathbf{J}^{(L-1)}\,\cdots\,\mathbf{J}^{(2)}.
\end{equation}
The first-layer weight gradient $\partial \calL/\partial \mathbf{W}^{(1)}$ is then
obtained by composing \eqref{eq:chain} with $\partial \mathbf{a}^{(1)}/\partial
\mathbf{W}^{(1)}$. The object of interest for gradient magnitude is therefore the
product of $L-1$ Jacobian matrices, each of which combines the spectral properties of
the weight matrices with the derivative of the activation function.

\subsection{Formal Analysis of Gradient Decay}

\begin{theorem}[Vanishing Gradient in the Contractive Regime]\label{thm:vanishing}
\citep{hochreiter1991,bengio1994}
  Let $\sigma:\R\to\R$ be Lipschitz with $\sup_z|\sigma'(z)| =: \kappa < \infty$ and
  suppose the weight matrices $\{\mathbf{W}^{(\ell)}\}$ satisfy
  $\norm{\mathbf{W}^{(\ell)}}_2 \le \rho$ for every layer, with $\rho\kappa < 1$.
  Then the backward Jacobian product satisfies
  \[
    \left\lVert \mathbf{J}^{(L)}\,\mathbf{J}^{(L-1)}\cdots\mathbf{J}^{(2)} \right\rVert_2
    \;\le\; \prod_{\ell=2}^{L} \norm{\mathbf{J}^{(\ell)}}_2
    \;\le\; (\rho\kappa)^{L-1}
    \;=:\; c^{L-1}
  \]
  with $c<1$, so the backpropagated gradient decays exponentially in depth.
\end{theorem}

\begin{proof}
  By submultiplicativity of the operator norm,
  $\norm{AB}_2 \le \norm{A}_2\norm{B}_2$. Since
  $\mathbf{J}^{(\ell)} = \operatorname{diag}(\sigma'(\mathbf{z}^{(\ell)}))\mathbf{W}^{(\ell)}$,
  \[
    \norm{\mathbf{J}^{(\ell)}}_2
    \le \norm{\operatorname{diag}(\sigma'(\mathbf{z}^{(\ell)}))}_2\,\norm{\mathbf{W}^{(\ell)}}_2
    \le \kappa\,\rho.
  \]
  Multiplying over the $L-1$ factors gives the result. For $\sigma=\tanh$ one has
  $\kappa = \sup_z|\sigma'(z)| = 1$, so the hypothesis reduces to the explicitly
  contractive condition $\rho<1$.
\end{proof}

The theorem isolates one mechanism of gradient decay: \emph{spectral contraction}.
A second mechanism is \emph{saturation}. Even when $\norm{\mathbf{W}^{(\ell)}}_2$
is not strictly below $1$, the factor
$\norm{\operatorname{diag}(\sigma'(\mathbf{z}^{(\ell)}))}_2$ becomes small if many
pre-activations lie in regions where $|\sigma'(z)| \ll 1$. For sigmoids and
hyperbolic tangents this occurs when $|z|$ is large. Hence vanishing gradients can
arise either because the linear part is contractive or because the nonlinearity is
locally saturated; historically, both mechanisms mattered.

\subsection{The Edge of Chaos: Mean-Field Analysis}

A sharper, distribution-level account of gradient propagation is given by the
mean-field theory of random initialisations. For
$W_{ij}^{(\ell)} \sim \mathcal{N}(0, \sigma_w^2/n)$ and
$b_i^{(\ell)} \sim \mathcal{N}(0, \sigma_b^2)$, define the order parameter
\[
  q^{(\ell)} \;=\; \frac{1}{n}\sum_{i=1}^n \E\!\left[(a_i^{(\ell)})^2\right].
\]

\begin{proposition}[Mean-Field Recursion]\label{prop:meanfield}
\citep{poole2016}
  To keep the variance bookkeeping precise, it is convenient to distinguish pre-activation and activation second moments. Define
  \[
    q_z^{(\ell)} := \frac{1}{n}\sum_{i=1}^n \E[(z_i^{(\ell)})^2],
    \qquad
    q_a^{(\ell)} := \frac{1}{n}\sum_{i=1}^n \E[(a_i^{(\ell)})^2].
  \]
  Under the random initialisation $W_{ij}^{(\ell)}\sim\mathcal{N}(0,\sigma_w^2/n)$ and $b_i^{(\ell)}\sim\mathcal{N}(0,\sigma_b^2)$, the mean-field recursion takes the form
  \[
    q_z^{(\ell+1)} \;=\; \sigma_w^2 q_a^{(\ell)} + \sigma_b^2,
    \qquad
    q_a^{(\ell+1)} \;=\; \int_{-\infty}^{\infty}
      \sigma\!\left(\sqrt{q_z^{(\ell+1)}}\,u\right)^2
      \frac{e^{-u^2/2}}{\sqrt{2\pi}}\,\mathrm{d}u.
  \]
  Equivalently, if one prefers a one-variable recursion, then the quantity denoted $q^{(\ell)}$ must be interpreted as a pre-activation variance rather than as an activation variance. At a fixed point $q_z^\ast = \sigma_w^2 q_a^\ast + \sigma_b^2$, the scalar
  \[
    \chi \;=\; \sigma_w^2 \int_{-\infty}^{\infty}
    \bigl(\sigma'(\sqrt{q_z^\ast}\,u)\bigr)^2
    \frac{e^{-u^2/2}}{\sqrt{2\pi}}\,\mathrm{d}u
  \]
  controls the linearised propagation of correlations and typical back-propagated signals.
  If $\chi < 1$, correlations contract and gradients vanish; if $\chi > 1$, they
  expand and gradients explode; if $\chi = 1$, the network is on the
  \emph{edge of chaos}.
\end{proposition}

This analysis shows that the vanishing (and exploding) gradient problems are
\emph{typical} for random initialisations that ignore the $\chi = 1$ constraint.
The set $\{(\sigma_w, \sigma_b) : \chi(\sigma_w, \sigma_b) = 1\}$ is a
one-dimensional curve in the two-dimensional parameter space: not pathological,
but not automatically reached either. Without principled initialisation schemes
(which did not exist until the work of \citet{glorot2010} and \citet{he2015}),
networks were unlikely to begin in the stable propagation regime, especially as
layer depth increased.

\section{Generalisation Theory: The Statistical Barriers}\label{sec:generalization}

\subsection{Uniform Convergence Bounds}

\begin{theorem}[Uniform Convergence Bound]\label{thm:vcgen}
\citep{blumer1989,shalev2014}
  Let $\calH$ be a binary hypothesis class with $\VCdim(\calH)=d$. Then with
  probability at least $1-\delta$ over a sample of size $m$,
  \[
    \sup_{h\in\calH} \abs{L_{\calD}(h)-L_S(h)}
    \;\leq\; C\sqrt{\frac{d\log(em/d)+\log(1/\delta)}{m}}
  \]
  for a universal constant $C>0$.
\end{theorem}

The point is qualitative rather than constant-sensitive. To make the right-hand
side at most $0.05$, requiring $C\sqrt{(d\log(em/d)+\log(1/\delta))/m} \leq 0.05$,
one needs sample size roughly $m \gtrsim 3200\,d$ (absorbing log-factors and
universal constants). For a network with $W = 10^4$ weights, using the rough
heuristic $\VCdim \approx W$~\citep{baum1989}, this requires $m \approx 3.2 \times 10^7$
labelled examples — far beyond what was available in the 1980s. The bound is
therefore non-explanatory for any practical dataset size of that era.

\subsection{The Bias--Variance Tradeoff}

For regression under squared loss, the mean squared error decomposes as:
\begin{equation}\label{eq:bv}
  \E\!\left[(\hat{f}(x) - f(x))^2\right]
  \;=\; \underbrace{\bigl(\E[\hat{f}(x)] - f(x)\bigr)^2}_{\text{Bias}^2(\hat{f})}
  + \underbrace{\Var[\hat{f}(x)]}_{\text{Variance}(\hat{f})}
  + \underbrace{\sigma_\varepsilon^2}_{\text{Irreducible noise}}.
\end{equation}

In Sobolev-type regularisation settings one often writes the complexity trade-off in the stylised form
\begin{itemize}[leftmargin=2em]
  \item $\text{Bias}^2 \asymp \lambda^{2s}$,
  \item $\text{Variance} \asymp d_\mathrm{eff}(\lambda)/m$,
\end{itemize}
where $\lambda$ is a regularisation scale, $s$ is a smoothness index, and
$d_\mathrm{eff}(\lambda) = \tr\!\bigl(\mathbf{H}(\mathbf{H} + \lambda\mathbf{I})^{-1}\bigr)$
denotes the effective dimension associated with the kernel or sample-covariance matrix
$\mathbf{H}$. The optimal complexity balances
\[
  \lambda^* = \argmin_\lambda \left[\lambda^{2s} + \frac{d_\mathrm{eff}(\lambda)}{m}\right].
\]
For neural networks the analogue of $\lambda$ may arise through explicit weight decay,
architectural restriction, or implicit mechanisms such as early stopping. The mapping is
therefore heuristic rather than literal. Even so, the lesson remains: when effective
model flexibility is large relative to sample size, variance dominates and out-of-sample
performance becomes unstable.

\subsection{Rademacher Complexity}

A tighter modern tool is Rademacher complexity.

\begin{definition}[Rademacher Complexity]
  The \emph{empirical Rademacher complexity} of $\calH$ with respect to
  sample $S = \{x_i\}_{i=1}^m$ is
  \[
    \hat{\mathfrak{R}}_S(\calH) \;=\; \E_{\bm{\sigma}}\!\left[
      \sup_{h \in \calH} \frac{1}{m}\sum_{i=1}^m \sigma_i h(x_i)
    \right],
  \]
  where $\bm{\sigma} = (\sigma_1,\ldots,\sigma_m)$ are i.i.d.\ Rademacher variables
  ($\Prob[\sigma_i = \pm 1] = 1/2$).
\end{definition}

\begin{theorem}[Rademacher Generalisation Bound]
  With probability at least $1-\delta$:
  \[
    L_{\calD}(h) \;\leq\; L_S(h) + 2\hat{\mathfrak{R}}_S(\calH)
    + \sqrt{\frac{\log(2/\delta)}{2m}}.
  \]
\end{theorem}

For a single-hidden-layer network with $k$ hidden units, weight matrix
$\mathbf{W}^{(1)} \in \R^{k \times d}$ with $\norm{\mathbf{W}^{(1)}}_F \leq B_1$,
output weights $\mathbf{w}^{(2)} \in \R^k$ with $\norm{\mathbf{w}^{(2)}}_1 \leq B_2$,
inputs satisfying $\norm{\mathbf{x}}_2 \leq X$ almost surely, and a $1$-Lipschitz
activation $\sigma$ with $\sigma(0)=0$, the Ledoux--Talagrand contraction inequality
\citep{bartlett2002} yields
\[
  \hat{\mathfrak{R}}_S(\calH) \;\leq\; \frac{B_1 B_2 X}{\sqrt{m}}.
\]
The $1/\sqrt{k}$ factor sometimes quoted in informal presentations is not a general
consequence of this setup. The reason $k$ disappears from the bound is that the
contraction inequality absorbs it: writing the network output as
$f(\mathbf{x}) = \mathbf{w}^{(2)\top}\sigma(\mathbf{W}^{(1)}\mathbf{x})$, one
applies contraction to the $\ell_1$-weighted sum over hidden units, which yields a
factor $\norm{\mathbf{w}^{(2)}}_1 \leq B_2$ regardless of $k$. Increasing $k$ while
keeping $B_1$ and $B_2$ fixed therefore does not worsen the bound, but at the cost
that those norm constraints become harder to satisfy in practice. Sharper bounds
incorporating activation geometry and width may yield $k$-dependent factors, but
these are not universal consequences of the setup above. For multilayer networks,
the analogous spectral-norm-based bounds \citep{bartlett2021} accumulate one
spectral-norm factor per layer, so the bound grows multiplicatively with depth ---
making classical generalisation guarantees essentially vacuous for deep networks
unless weight norms are tightly controlled.

\section{Resolutions: Mathematical Breakthroughs That Ended the Winters}\label{sec:resolution}

The modern era is often associated with the deep-learning synthesis catalogued by
\citet{lecun2015} and \citet{goodfellow2016}, in which architectural innovations,
algorithmic refinements, and the availability of large labelled corpora combine to
push performance past the regimes in which the barriers of
Sections~\ref{sec:perceptron}--\ref{sec:generalization} bind decisively. We organise
the underlying mathematical resolutions by the barrier each one addresses.

\subsection{Universal Approximation and Depth Separation}

The first key insight was that sufficiently wide networks overcome the
representational barrier.

\begin{theorem}[Universal Approximation]\label{thm:uat}
\citep{cybenko1989,hornik1991,leshno1993}
  Let $\sigma:\R\to\R$ be continuous. Single-hidden-layer networks with activation
  $\sigma$ are universal approximators on compact subsets of $\R^d$ --- meaning
  that for every $f \in C([0,1]^d)$ and $\varepsilon > 0$ there exist $N \in \N$,
  coefficients $\alpha_i \in \R$, biases $b_i \in \R$, and weights
  $\mathbf{w}_i \in \R^d$ with
  \[
    \sup_{\mathbf{x} \in [0,1]^d}
    \left\lvert f(\mathbf{x}) - \sum_{i=1}^N \alpha_i\,\sigma\!\left(\inner{\mathbf{w}_i}{\mathbf{x}} + b_i\right)\right\rvert
    \;<\; \varepsilon
  \]
  --- if and only if $\sigma$ is not a polynomial. The sufficiency results of
  \citet{cybenko1989} (for sigmoidal $\sigma$) and \citet{hornik1991} (for bounded
  non-constant $\sigma$) were sharpened to this necessary-and-sufficient form by
  \citet{leshno1993}.
\end{theorem}

Universal approximation is therefore a statement about expressivity in principle,
not about efficient representation. A shallow network may approximate any continuous
function on a compact domain, yet require enormous width to do so. The deeper insight
is that \emph{depth} can change representational efficiency dramatically, a fact
foreshadowed by \citet{delalleau2011} for sum--product networks and made precise for
ReLU networks by the depth-separation results of the mid-2010s.

\begin{proposition}[Universal Approximation versus Efficient Representation]
Universal approximation and depth separation are not in tension. The former says that
shallow networks are dense in suitable function spaces; the latter says that for some
function families, achieving a fixed approximation error with a shallow network may require
exponential width, whereas a deeper network achieves the same error with polynomial size.
The issue is therefore not expressivity in principle, but representational efficiency.
\end{proposition}

\begin{theorem}[Depth Separation for ReLU Networks]\label{thm:depth}
\citep{telgarsky2016, eldan2016}
  \begin{enumerate}
    \item \textbf{(Telgarsky, 2016.)} For any $k \in \N$, there exists a function
      computable by a ReLU network of depth $3k$ and $\mathcal{O}(1)$ width that
      cannot be approximated to constant error by any network of depth $k$ and
      sub-exponential width $2^{k^{1/3}}$.
    \item \textbf{(Eldan \& Shamir, 2016.)} There exists a radial function
      $f:\R^d \to \R$ that is efficiently computable by a 3-layer network of
      polynomial size but requires exponential size (in $d$) to approximate
      to constant error by any 2-layer network.
  \end{enumerate}
\end{theorem}

\begin{remark}[Linear-Region Counting]\label{rem:linearregion}
\citep{montufar2014}
  An earlier and complementary expressivity argument is the linear-region count of
  \citet{montufar2014}, which predates the approximation-theoretic separations
  above. They showed that depth-$k$ ReLU networks of width $n$ can produce at least
  \[
    \Omega\!\left(\left(\lfloor n/d \rfloor\right)^{(k-1)d} \cdot n\right)
  \]
  linear regions --- exponentially more than depth-2 networks of the same width.
  This expressivity gap is consistent with, but does not directly imply, the
  approximation-theoretic separations of Theorem~\ref{thm:depth}, which concern
  approximation error for specific function families.
\end{remark}

\subsection{The Representer Theorem and Kernel Methods}

The 1990s saw the emergence of kernel methods~\citep{boser1992,cortes1995}, which
circumvent intractable weight-space search by recasting optimisation in a
reproducing kernel Hilbert space (RKHS).

\begin{theorem}[Representer Theorem]\label{thm:representer}
\citep{scholkopf2001}
  Let $\calH_k$ be an RKHS with kernel $k$. For any loss function $\ell$ and
  regularisation $\Omega(\norm{f}_{\calH_k})$ with $\Omega$ strictly monotone
  increasing, the solution to
  \[
    \min_{f \in \calH_k} \;\frac{1}{m}\sum_{i=1}^m \ell(y_i, f(x_i)) + \Omega(\norm{f}_{\calH_k})
  \]
  takes the form $f^*(\cdot) = \sum_{i=1}^m \alpha_i k(x_i, \cdot)$.
\end{theorem}

\begin{proof}
  Decompose any $f \in \calH_k$ as $f = f_\| + f_\perp$ where
  $f_\| \in \text{span}\{k(x_i,\cdot)\}_{i=1}^m$ and $f_\perp$ is orthogonal
  to this span. Then $f(x_i) = \inner{f}{k(x_i,\cdot)}_{\calH_k} = f_\|(x_i)$
  for all $i$, so $f_\perp$ affects the regularisation term only (since
  $\norm{f}^2 = \norm{f_\|}^2 + \norm{f_\perp}^2 \geq \norm{f_\|}^2$).
  Setting $f_\perp = 0$ therefore cannot increase the loss while it strictly
  decreases the regularisation, contradicting optimality of any $f_\perp \neq 0$.
\end{proof}

The kernel trick $k(\mathbf{x},\mathbf{z}) = \phi(\mathbf{x})^\top\phi(\mathbf{z})$
allows implicit computation of dot products in potentially infinite-dimensional
feature spaces, without explicitly computing $\phi$. For the Gaussian RBF:
$k(\mathbf{x},\mathbf{z}) = \exp(-\norm{\mathbf{x}-\mathbf{z}}^2/2\sigma^2)$,
the corresponding feature map is infinite-dimensional.

\begin{remark}[The Neural Tangent Kernel]\label{rem:ntk}
\citep{jacot2018}
  A remarkable theoretical bridge between kernel methods and modern deep learning
  is the Neural Tangent Kernel (NTK). For a network $f_{\bm{\theta}}$ with parameter
  vector $\bm{\theta}$, define
  \[
    K_{\mathrm{NTK}}(\mathbf{x}, \mathbf{z}) \;=\;
    \left\langle \frac{\partial f_{\bm{\theta}}(\mathbf{x})}{\partial \bm{\theta}},\;
    \frac{\partial f_{\bm{\theta}}(\mathbf{z})}{\partial \bm{\theta}} \right\rangle.
  \]
  \citet{jacot2018} showed that in the infinite-width limit, $K_{\mathrm{NTK}}$
  converges to a deterministic kernel at initialisation and remains constant throughout
  gradient descent training. In this limit, training the network is equivalent to
  kernel regression with $K_{\mathrm{NTK}}$, providing an exact RKHS interpretation
  of infinite-width neural networks. The NTK therefore closes a conceptual loop:
  the representational barriers of the perceptron era (Section~\ref{sec:perceptron})
  are circumvented by depth and nonlinearity, and the resulting function class
  admits a kernel description in the lazy-training regime. Whether practical finite-width
  networks operate near this kernel regime remains an active research question.
\end{remark}

\subsection{Optimisation: Activations, Initialisation, and Residual Connections}

The vanishing gradient problem was resolved by a combination of:

\paragraph{(i) ReLU activations.} The rectified linear unit $\sigma(z)=\max\{0,z\}$
avoids the small-derivative saturation that makes sigmoid and $\tanh$ networks
especially vulnerable to gradient decay in their tails. This does not mean that
gradient gating disappears: $\sigma'(z)=\mathbf{1}_{\{z>0\}}$, so inactive units can
still block gradient flow. The practical gain is that \emph{active} units contribute
a derivative of order one rather than an exponentially small factor, making stable
signal propagation compatible with variance-preserving initialisation. Combined with
the He scheme \citep{he2015}, the typical spectral factor $\rho\kappa$ entering
Theorem~\ref{thm:vanishing} is brought close to unity, so gradients propagate without
exponential decay across realistic depths.

\paragraph{(ii) He/Xavier Initialisation \citep{glorot2010,he2015}.}
In the scaling $W_{ij}^{(\ell)}\sim\mathcal{N}(0,\sigma_w^2/n)$ used above,
He initialisation for ReLU corresponds to $\sigma_w^2 = 2$ (so that
$\mathrm{Var}(W_{ij})=2/n$ with $n$ the fan-in), which places the network at the
edge of chaos $\chi = 1$ at initialisation and yields well-conditioned gradient
flow. Xavier initialisation corresponds to $\sigma_w^2 = 1$ and is the analogous
variance-preserving choice for $\tanh$ networks.

\paragraph{(iii) Residual connections \citep{he2016}.}
Adding skip connections $\mathbf{a}^{(\ell+1)} = \sigma(\mathbf{W}^{(\ell)}\mathbf{a}^{(\ell)}) + \mathbf{a}^{(\ell)}$
modifies the per-layer Jacobian to
\[
  \mathbf{J}^{(\ell)}_{\mathrm{res}}
  \;=\; \frac{\partial \mathbf{a}^{(\ell+1)}}{\partial \mathbf{a}^{(\ell)}}
  \;=\; \mathbf{I} + \operatorname{diag}(\sigma'(\mathbf{z}^{(\ell)}))\mathbf{W}^{(\ell)}.
\]
The presence of $\mathbf{I}$ does \emph{not} guarantee a unit lower bound on singular
values in general: if $\mathbf{A} = \operatorname{diag}(\sigma')\mathbf{W}$ has a
singular value close to $-1$ (possible when $\mathbf{W}$ has a large negative
eigencomponent), then $\mathbf{I} + \mathbf{A}$ may have singular values well below
$1$. A simple scalar instance: $A = -\tfrac{1}{2}$ gives $I + A = \tfrac{1}{2}$.
The practical benefit of residuals is more subtle: they create \emph{additive
gradient paths} through the network --- the gradient of the loss with respect to
early-layer parameters receives a direct contribution
$\partial \calL/\partial \mathbf{a}^{(L)}$ bypassing all intermediate weight
matrices. This bypassed path is not subject to the multiplicative decay
in~\eqref{eq:chain}, and empirically dominates when intermediate layers are small or
near-zero, maintaining meaningful gradient signal even at extreme depth
\citep{he2016}.

\subsection{The Double Descent Phenomenon}

A later development, important for retrospective interpretation though not a direct historical cause of the winters themselves, is the double-descent phenomenon. It helps explain why classical statistical intuition was incomplete in highly overparameterised regimes.

\begin{theorem}[Asymptotic Risk in Overparameterised Regression]\label{thm:doubledescent}
\citep{hastie2022}
  Consider the minimum-norm least-squares interpolant
  $\hat{\mathbf{w}} = \mathbf{X}^\dagger\mathbf{y}$ for design matrix
  $\mathbf{X} \in \R^{m \times p}$ with i.i.d.\ isotropic Gaussian rows, and labels
  $\mathbf{y} = \mathbf{X}\mathbf{w}^* + \bm{\varepsilon}$,
  $\varepsilon_i \overset{\mathrm{iid}}{\sim} \mathcal{N}(0,\sigma^2)$.
  Let $\gamma = p/m$ (the overparameterisation ratio). As $m, p \to \infty$ with
  $\gamma > 1$ fixed, the test risk decomposes as
  \[
    \E\!\left[L_\calD(\hat{f})\right]
    \;\to\;
    \underbrace{\norm{\mathbf{w}^*}^2\!\left(1 - \frac{1}{\gamma}\right)}_{\text{Bias}^2}
    \;+\;
    \underbrace{\sigma^2 \cdot \frac{1}{\gamma - 1}}_{\text{Variance}}.
  \]
  As $\gamma \to \infty$, the variance term $\sigma^2/(\gamma-1)$ vanishes, but the bias term
  converges to $\norm{\mathbf{w}^*}^2$ in general. Thus the test risk does \emph{not}
  converge to zero in the isotropic Gaussian model unless the signal itself shrinks in a way
  that makes the bias vanish.
\end{theorem}

\begin{remark}[Interpreting Double Descent]
  The original empirical observation of \citet{belkin2019} --- that test risk
  decreases again after the interpolation threshold $p = m$ --- is robust and
  important. The precise asymptotic formula above shows that the behaviour in the
  overparameterised regime depends on both the signal strength $\norm{\mathbf{w}^*}$
  and the noise $\sigma^2$, not on noise alone. The practical success of
  overparameterised networks is better understood through the lens of
  \emph{implicit regularisation}: gradient descent preferentially selects low-complexity
  interpolating solutions in many settings \citep{zhang2017,neyshabur2017,gunasekar2018,bartlett2002,bartlett2021}.
  The key conceptual point stands: classical bias--variance intuition, which treats
  overparameterisation as uniformly harmful, is incomplete.
\end{remark}

\subsection{Modern Theory: Four Post-Hoc Lenses on the Barrier Triple}\label{sec:modern}

\paragraph{Methodological note.} The four items below are epistemically heterogeneous. Neural scaling laws and neural collapse are primarily empirical regularities with partial theoretical explanation; the lottery ticket hypothesis is, as its name indicates, a hypothesis supported by substantial experiments rather than a general theorem; implicit-bias theory contains theorem-level results, but mostly in restricted model classes. They are grouped here not as coequal formal results, but as post-hoc explanatory lenses on different coordinates of the barrier triple.

The four theoretical developments below are not themselves the historical
interventions that ended the winters; those were, more prosaically, multilayer
architectures, non-saturating activations, improved initialisation, large labelled
corpora, and hardware scaling. The developments discussed here are later
\emph{theoretical lenses} on the modern era: each illuminates a specific component
of the barrier triple $(R, C, S)$ from Definition~\ref{def:barrier} and helps
explain, after the fact, why contemporary large-scale systems tend to sit outside
the joint-binding regime of Proposition~\ref{thm:fragility}.

\subsubsection*{Neural Scaling Laws (S-barrier)}

\citet{kaplan2020} demonstrated that the test loss of large language models follows
precise power-law relationships in model size $N$, dataset size $D$, and compute
budget $C_{\mathrm{flop}}$:
\[
  L(N,D) \;\approx\; \left(\frac{N_c}{N}\right)^{\alpha_N}
                    + \left(\frac{D_c}{D}\right)^{\alpha_D}
                    + L_\infty,
\]
with empirically fitted exponents $\alpha_N \approx 0.076$,
$\alpha_D \approx 0.095$, and irreducible loss $L_\infty$.
\citet{hoffmann2022} refined this picture with a compute-optimal (Chinchilla) result:
for a fixed FLOP budget $C_{\mathrm{flop}}$, the optimal allocation satisfies
$N^* \propto C_{\mathrm{flop}}^{a}$ and $D^* \propto C_{\mathrm{flop}}^{b}$ with
empirically estimated exponents $a, b \approx 0.5$ (point estimates fall in the
range $0.46$--$0.54$ depending on the fitting approach), i.e.\ model size and data
should scale in roughly equal proportion.

\begin{remark}[Scaling Laws and the Statistical Barrier]
The classical statistical barrier $S$ of Definition~\ref{def:barrier} asserts that
finite-sample guarantees require unrealistic data volumes. Scaling laws give this
assertion a quantitative inversion: they specify \emph{how many} samples and
parameters are needed to reach a target loss, under the assumption that
both are available. They do not eliminate the barrier — the Stone minimax rate
(Theorem~\ref{thm:stone}) still governs pointwise estimation — but they show that
the barrier is \emph{metrically navigable}: smooth power-law curves, rather than
hard walls, control performance. This is precisely the shift that distinguishes the
modern era from the second winter, when VC-style bounds gave only vacuous guarantees
with no actionable guidance on scale.
\end{remark}

\subsubsection*{Neural Collapse (R-barrier)}

\citet{papyan2020} identified a striking geometric phenomenon in the terminal phase
of deep-network training. As the loss approaches zero, the within-class variability
of last-layer features $\mathbf{h}_i^{(c)}$ collapses:
\[
  \frac{1}{n_c}\sum_{i=1}^{n_c}\norm{\mathbf{h}_i^{(c)} - \bm{\mu}_c}^2
  \;\to\; 0,
\]
and the class means $\bm{\mu}_1,\dots,\bm{\mu}_K$ converge to a
\emph{simplex equiangular tight frame} (ETF):
\[
  \cos\angle(\bm{\mu}_c - \bm{\mu}_0,\; \bm{\mu}_{c'} - \bm{\mu}_0)
  = -\frac{1}{K-1}, \quad c \neq c',
\]
the maximally spread configuration in $\R^d$ for $K$ unit vectors. Simultaneously,
the classifier weight vectors $\mathbf{w}_c$ align with the class means.

\begin{remark}[Neural Collapse and the Representational Barrier]
The representational barrier $R$ of Definition~\ref{def:barrier} captures the
inability of a model class to realise target functions. Neural collapse implies that
deep networks do more than find \emph{some} representation: they converge to a
geometry-optimal one — the ETF — that maximises inter-class separation given the
ambient dimension. This offers a partial geometric lens on the puzzle of generalisation in deep networks: the terminal-phase representation is structured rather than arbitrary, and under symmetric Gaussian class-conditional models the ETF geometry coincides with that of Bayes-optimal classifiers. The evidence base is, however, largely empirical and restricted to the final training phase; neural collapse should be read as a suggestive regularity that aligns with the representational component $R$ of Definition~\ref{def:barrier}, not as a derivation of generalisation from first principles.
\end{remark}

\subsubsection*{The Lottery Ticket Hypothesis (C-barrier)}

\citet{frankle2019} proposed the \emph{lottery ticket hypothesis}: a randomly
initialised dense network $f(\mathbf{x};\bm{\theta})$ of size $n$ contains a
sparse subnetwork (a ``winning ticket'') $f(\mathbf{x};\bm{\theta}_0 \odot \mathbf{m})$,
where $\mathbf{m}\in\{0,1\}^n$ is a binary mask, such that when trained from the
\emph{original} initialisation $\bm{\theta}_0 \odot \mathbf{m}$ (not from a fresh
random start), it matches the full network's accuracy in at most the same number of
iterations. Formally, with target accuracy $a^\ast$ and iteration budget $j^\ast$
for the full network:
\[
  \exists\,\mathbf{m}:\;|\mathbf{m}|_1 \ll n,\quad
  \text{accuracy}\bigl(f(\cdot;\bm{\theta}_{j^\ast}\odot\mathbf{m})\bigr) \;\ge\; a^\ast.
\]

\begin{remark}[Lottery Tickets and the Computational Barrier]
The computational barrier $C$ of Definition~\ref{def:barrier} rests on the worst-case intractability of finding correct weights (Theorem~\ref{thm:blumrivest}). The lottery ticket hypothesis — still a hypothesis, supported by substantial but not universal empirical evidence and sensitive to architecture, dataset, and pruning protocol — does not contradict this hardness result: worst-case intractability survives. What it \emph{suggests}, at the level of an empirical regularity, is that in the regimes where winning tickets have been observed, useful sparse structure is already latent in the dense initialisation, so that practical training resembles identification of a pre-existing subnetwork rather than search through an exponential space from scratch. Read this way, lottery tickets sit alongside the resolutions of Section~\ref{sec:resolution} as a post-hoc lens on why the $C$-barrier does not bind in practice for the architectures studied, rather than as a theorem about trainability in general.
\end{remark}

\subsubsection*{Implicit Bias Theory (S-barrier)}

\citet{soudry2018} proved the first rigorous result on the implicit bias of gradient
descent: for linearly separable data, gradient descent on the logistic loss with any
step-size schedule satisfying $\sum_t \eta_t = \infty$, $\sum_t \eta_t^2 < \infty$
converges in direction to the \emph{hard-margin SVM solution} --- the
maximum-margin classifier $\hat{\mathbf{w}}/\norm{\hat{\mathbf{w}}}$:
\[
  \frac{\mathbf{w}(t)}{\norm{\mathbf{w}(t)}} \;\to\; \hat{\mathbf{w}}_{\mathrm{SVM}}
  \quad \text{as } t \to \infty.
\]
\citet{gunasekar2018} extended this analysis to matrix factorisation and linear
convolutional networks, showing that gradient descent imposes a nuclear-norm
minimisation bias in the matrix case and a structured frequency-domain sparsity bias
in the convolutional case.

\begin{remark}[Implicit Bias and the Statistical Barrier]
The statistical barrier $S$ of Definition~\ref{def:barrier} is most acute because
classical bounds (VC, Rademacher) apply to the \emph{worst-case} hypothesis in a
class, while gradient descent does not explore the class uniformly. The implicit
bias results of \citet{soudry2018} and \citet{gunasekar2018} make this precise: the
algorithm itself performs implicit regularisation, converging to minimum-complexity
solutions (maximum margin, minimum nuclear norm) that happen to generalise well.
This bridges the gap between the vacuous worst-case statistical bounds of
Section~\ref{sec:generalization} and the empirical success of unregularised gradient
descent on overparameterised models. The statistical barrier is not eliminated --- it
is circumvented by the algorithm's inductive bias, which provides an effective
regulariser that classical theory did not anticipate.
\end{remark}

\paragraph{Connection to the barrier model.} Each of the four developments above offers a post-hoc lens on a specific component of the barrier triple from Definition~\ref{def:barrier}: scaling laws quantify the regime in which the statistical barrier $S$ becomes metrically navigable; neural collapse describes a geometric structure that partially accounts for why the representational barrier $R$ does not bite in practice; lottery tickets reframe the computational barrier $C$ from exhaustive search to latent structure identification; and implicit bias theory shows how the statistical barrier $S$ can be circumvented by the algorithm's inductive geometry rather than by tighter worst-case bounds. Proposition~\ref{thm:fragility} characterises joint binding of $(R, C, S)$ as the structural signature of paradigm fragility; these four lenses provide complementary explanations for why the modern era, shaped primarily by architectural and hardware progress, appears to avoid that regime.

\section{A Mathematical Taxonomy of AI Winter Causes}\label{sec:taxonomy}

\subsection{A Structural Synthesis}

We now state explicitly the structural logic implicit in the previous sections,
unifying Definition~\ref{def:barrier} with the barriers established in
Sections~\ref{sec:perceptron}--\ref{sec:generalization}.

\begin{proposition}[Joint-Binding Condition for Structural Fragility]\label{thm:fragility}
Let $\mathcal{A}$ be an architecture class, $\mathcal{T}$ a family of tasks equipped
with a task distribution of interest, and fix a data budget $n$ and compute budget
$\kappa$. Suppose the following three conditions hold for a target error
$\varepsilon>0$:
\begin{enumerate}[label=(\roman*)]
  \item \textbf{Representational limitation:} there exists a subset
    $\mathcal{T}_R \subseteq \mathcal{T}$ of positive measure whose target functions
    are not $\varepsilon$-approximable by any element of $\mathcal{A}$;
  \item \textbf{Computational hardness:} locating an $\varepsilon$-optimal hypothesis
    in $\mathcal{A}$ is worst-case intractable, or its cost grows super-polynomially
    in the relevant problem parameters, so that no algorithm respecting budget
    $\kappa$ is known to succeed reliably;
  \item \textbf{Statistical fragility:} the finite-sample guarantees available for
    $\mathcal{A}$ require a sample size that exceeds $n$ by a super-constant factor
    at the scales of interest.
\end{enumerate}
If at least two of these conditions bind simultaneously under the budget
$(n,\kappa)$, then any learning system based on $\mathcal{A}$ must exhibit at
least one of the following pathologies on $\mathcal{T}$:
\begin{enumerate}[label=(\alph*)]
  \item irreducible approximation error on a positive-measure subset,
  \item computational infeasibility of reliable training within budget $\kappa$,
  \item generalisation guarantees that remain vacuous at sample size $n$.
\end{enumerate}
We call any paradigm satisfying the hypotheses \emph{structurally fragile}
with respect to the budget $(\mathcal{A},\mathcal{T},n,\kappa)$.
\end{proposition}

\begin{proof}[Synthesis]
The statement reorganises the preceding sections rather than establishing a new
lower bound. If (i) binds, approximation error is bounded away from zero on
$\mathcal{T}_R$ regardless of optimisation or data (Section~\ref{sec:perceptron}).
If (ii) binds, expressive hypotheses may exist in principle but cannot be located
reliably within budget $\kappa$ (Section~\ref{sec:complexity}). If (iii) binds,
empirical fit does not transfer into controlled generalisation at sample size $n$
(Section~\ref{sec:generalization}). When any two bind jointly, the paradigm has no
simultaneous route to expressivity, trainability, and sample-efficient
generalisation within the budget, so at least one of (a)--(c) must occur.
\end{proof}

\begin{remark}[Framework status]
Proposition~\ref{thm:fragility} should be read as a budgeted synthesis rather than as a new impossibility theorem. Its force is logical: once a paradigm requires simultaneous expressivity, trainability, and non-vacuous finite-sample control, the joint failure of any two of these conditions under the same resource budget is enough to render the paradigm structurally unstable. The proposition therefore packages the earlier obstructions into a single fragility criterion instead of asserting a fresh lower bound. It makes no claim about the historical sufficiency of these mathematical conditions: economic, institutional, and engineering factors lie outside its scope. Structural fragility is a \emph{necessary condition} for the kind of paradigm collapse associated with historical AI winters, not a sufficient one.
\end{remark}

\subsection{Barrier Diagram}

\begin{figure}[h!]
\centering
\begin{tikzpicture}[font=\footnotesize, line width=0.4pt]

  \coordinate (R) at (-2.95,-1.70);
  \coordinate (C) at ( 2.95,-1.70);
  \coordinate (S) at ( 0.00, 3.41);

  \draw (R) -- (C) -- (S) -- cycle;

  \node[below=3pt of R, align=center] {$R$\\\textit{representational}};
  \node[below=3pt of C, align=center] {$C$\\\textit{computational}};
  \node[above=3pt of S, align=center] {$S$\\\textit{statistical}};

  \node[below=2pt]                                at ($(R)!0.5!(C)$) {$R \cap C$ \textit{fragile}};
  \node[rotate=60,  anchor=south, yshift=2pt]     at ($(R)!0.5!(S)$) {$R \cap S$ \textit{fragile}};
  \node[rotate=-60, anchor=south, yshift=2pt]     at ($(C)!0.5!(S)$) {$C \cap S$ \textit{fragile}};

  \node[align=center] at (0, 0.15)
    {\textit{joint binding}\\$R \cap C \cap S$};

\end{tikzpicture}
\caption{Structural interpretation of Proposition~\ref{thm:fragility}.
Each vertex is a single barrier from Definition~\ref{def:barrier}; each edge is a
regime in which two barriers bind jointly; the interior is the joint-binding
regime in which all three bind under the same data and compute budget.}
\label{fig:barrier-triangle}
\end{figure}
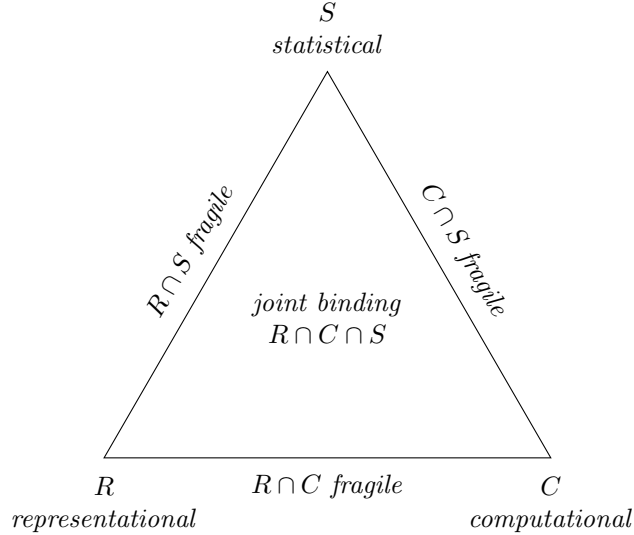

\subsection{Comparative Taxonomy}

Table~\ref{tab:taxonomy} summarises the main barriers discussed in the article. The final column refers to later mitigations rather than complete eliminations. In each case, the point is not that a single theorem ``caused'' an historical winter, but that the dominant methods of the time were exposed to a structurally important formal bottleneck.

\begin{table}[htbp]
\centering
\caption{Formal barriers associated with the AI winters and later mitigations.}
\label{tab:taxonomy}
\renewcommand{\arraystretch}{1.2}
\footnotesize
\setlength{\tabcolsep}{3pt}
\setlength{\emergencystretch}{3em}
\hbadness=10000
\begin{tabular}{@{}>{\raggedright\arraybackslash}p{1.7cm}
                >{\raggedright\arraybackslash}p{1.9cm}
                >{\raggedright\arraybackslash}p{2.6cm}
                >{\raggedright\arraybackslash}p{2.9cm}
                >{\raggedright\arraybackslash}p{2.9cm}@{}}
\toprule
\textbf{Barrier} & \textbf{Primary relevance} & \textbf{Formal object} & \textbf{Mathematical implication} & \textbf{Later mitigation} \\
\midrule
Repre\-sen\-ta\-tional & First winter & Single-layer threshold classifiers & Non-linearly separable functions (e.g.\ XOR) are not representable & Multi\-layer architectures; universal approxi\-mation; depth separation \\
Compu\-tational & Mostly first; partly both & Exact training of threshold networks; symbolic search trees & Worst-case hardness and exponential search growth & Stochastic optimi\-sation; heuristics; hardware scale; approximate learning \\
Statistical & Mostly second; also retrospective & VC dimension and uniform convergence & Classical bounds become loose or vacuous in flexible regimes & Margin methods; norm control; implicit regular\-isation; interpolation theory \\
Optimi\-sation & Second winter & Products of Jacobians in deep networks & Exponential decay or explosion of gradients under saturating dynamics & ReLU activations; careful initial\-isation; residual connections \\
Dimen\-sionality & Both & Minimax rates for nonparametric estimation & Sample complexity deterio\-rates rapidly with ambient dimension & Learned hierarch\-ical represen\-tations; task structure; invariances \\
Symbolic knowledge & Second winter & Rule bases and persistence constraints & Brittleness and combina\-torial maintenance burden & Representation learning and data-driven feature formation \\
\bottomrule
\end{tabular}
\normalsize
\end{table}

\section{Conclusion}\label{sec:conclusion}

This article has presented a mathematical interpretation of AI winter that is
narrower and more defensible than a monocausal historical thesis: not that
specific theorems caused historical retrenchment, but that the leading paradigms of
the relevant eras were already structurally exposed before external shocks arrived.
The central claim --- formalised in Definition~\ref{def:barrier} and
Proposition~\ref{thm:fragility} --- is that several of the most visible
disappointments of early AI aligned with genuine formal bottlenecks. Single-layer
perceptrons were representationally limited. Exact training of small threshold
networks was already worst-case hard. Symbolic search suffered exponential growth.
High-dimensional smooth-function estimation obeyed unfavourable minimax rates. Deep
gradient propagation under saturating nonlinearities was unstable. Classical
finite-sample generalisation theory, meanwhile, did not explain why large flexible
models should perform well in practice. None of these facts by itself determines an
historical outcome, but together they help explain why major paradigms proved
brittle when ambitious claims met finite compute, limited data, and real
engineering constraints. That is the precise sense in which formal barriers
contributed to winter: not as sufficient causes, but as mathematically articulable
sources of fragility.

The second lesson is equally important. The end of an AI winter did not come from the disappearance of these barriers, but from the discovery of architectures and mathematical viewpoints that mitigated them. Multilayer networks overcame the strict linear-separability bottleneck. Kernel and margin methods supplied new forms of statistical control. Improved initialisation, non-saturating activations, and residual architectures altered optimisation geometry. Modern representation learning exploited compositional structure and invariances to soften the curse of dimensionality in practice. Recent theory on interpolation and implicit regularisation has further clarified why overparameterisation need not be fatal, even if many foundational questions remain open.

The broader implication is methodological. AI advances become durable when claims of intelligence are matched by explicit analysis of representational, computational, statistical, and optimisation constraints. The history of AI winter is therefore not only a cautionary tale about hype but also a reminder that progress in artificial intelligence depends on mathematics as much as on engineering.

\paragraph{Limitations.} The article has focused on formal bottlenecks and has not attempted a complete intellectual history of the two winters. It has also treated later developments mainly through theorem-level lenses, leaving aside many empirical and institutional factors. The account should therefore be read as a mathematical anatomy of fragility, not as a comprehensive explanation of every historical event.

\bibliographystyle{plainnat}
\bibliography{ai_winter_full}

\end{document}